\newcommand{\mb}[1]{\mathbf{#1}}
\documentclass{article}
\usepackage[utf8]{inputenc}
\usepackage{amsmath}
\usepackage{xcolor}
\usepackage{graphicx}
\usepackage{caption}
\usepackage{subcaption}
\usepackage{floatrow}

\usepackage[nonatbib, preprint]{neurips_2020}

\usepackage[utf8]{inputenc} % allow utf-8 input
\usepackage[T1]{fontenc}    % use 8-bit T1 fonts

\usepackage[colorlinks]{hyperref}       % hyperlinks
\hypersetup{
    colorlinks = true,
    linkcolor = blue,
    anchorcolor = blue,
    citecolor = blue,
    filecolor = blue,
    urlcolor = blue
}
     
\usepackage{amsthm}
\newtheorem{theorem}{Theorem}[section]

\title{A bio-inspired bistable recurrent cell allows for long-lasting memory}
\author{Nicolas Vecoven \\ University of Liège \\ nvecoven@uliege.be \And Damien Ernst\\ University of Liège  \And Guillaume Drion\\ University of Liège}
\date{February 2020}

\begin{document}

\maketitle

\begin{abstract}
Recurrent neural networks (RNNs) provide state-of-the-art performances in a wide variety of tasks that require memory. These performances can often be achieved thanks to gated recurrent cells such as gated recurrent units (GRU) and long short-term memory (LSTM). Standard gated cells share a layer internal state to store information at the network level, and long term memory is shaped by network-wide recurrent connection weights.
Biological neurons on the other hand are capable of holding information at the cellular level for an arbitrary long amount of time through a process called bistability. Through bistability, cells can stabilize to different stable states depending on their own past state and inputs, which permits the durable storing of past information in neuron state. In this work, we take inspiration from biological neuron bistability to embed RNNs with long-lasting memory at the cellular level. This leads to the introduction of a new bistable biologically-inspired recurrent cell that is shown to strongly improves RNN performance on time-series which require very long memory, despite using only cellular connections (all recurrent connections are from neurons to themselves, i.e. a neuron state is not influenced by the state of other neurons). Furthermore, equipping this cell with recurrent neuromodulation permits to link them to standard GRU cells, taking a step towards the biological plausibility of GRU.
\end{abstract}

\section{Introduction} 
Recurrent neural networks (RNNs) have been widely used in the past years, providing excellent performances on many problems requiring memory such as e.g. sequence to sequence modeling, speech recognition, and neural translation. These achievements are often the result of the development of the long short-term memory (LSTM \cite{lstm}) and gated recurrent units (GRU \cite{gru}) recurrent cells, which allow RNNs to capture time-dependencies over long horizons. Despite all the work analyzing the performances of such cells \cite{rnnevaluation}, recurrent cells remain predominantly black-box models. There has been some advance in understanding the dynamical properties of RNNs as a whole from a non-linear control perspective (\cite{control_crnn}), but little has been done in understanding the underlying system of recurrent cells themselves. Rather, they have been built for their robust mathematical properties when computing gradients with back-propagation through time (BPTT). Research on new recurrent cells is still ongoing and, building up on LSTM and GRU, recent works have proposed other types of gated units (\cite{mgu}, \cite{gruvariants}, \cite{gated_orthogonal}). In addition, an empirical search over hundreds of different gated architectures has been carried in \cite{empiricalrnn}.

In parallel, there has been an increased interest in assessing the biological plausibility of neural networks. There has not only been a lot of interest in spiking neural networks (\cite{spikingdeep, spikingdeep2, lstminspiking}), but also in reconciling more traditional deep learning models with biological plausibility (\cite{towardsplausible,plausiblernnlearning,biologicallyinspiredbptt}). RNNs are a promising avenue for the latter (\cite{recurrentneuroscience}) as they are known to provide great performances from a deep learning point of view while theoretically allowing a discrete dynamical simulation of biological neurons. 

RNNs combine simple cellular dynamics and a rich, highly recurrent network architecture. The recurrent network architecture enables the encoding of complex memory patterns in the connection weights. These memory patterns rely on global feedback interconnections of large neuronal populations. Such global feedback interconnections are difficult to tune, and can be a source of vanishing or exploding gradient during training, which is a major drawback of RNNs. In biological networks, a significant part of advanced computing is handled at the cellular level, mitigating the burden at the network level. Each neuron type can switch between several complex firing patterns, which include e.g. spiking, bursting, and bistability. In particular, bistability is the ability for a neuron to switch between two stable outputs depending on input history. It is a form of cellular memory (\cite{marder}). 

In this work, we propose a new biologically motivated bistable recurrent cell (BRC), which embeds classical RNNs with local cellular memory rather than global network memory. More precisely, BRCs are built such that their hidden recurrent state does not directly influence other neurons (i.e. they are not recurrently connected to other cells). To make cellular bistability compatible with the RNNs feedback architecture, a BRC is constructed by taking a feedback control perspective on biological neuron excitability (\cite{drioncontrole}). This approach enables the design of biologically-inspired cellular dynamics by exploiting the RNNs structure rather than through the addition of complex mathematical functions. 

To test the capacities of cellular memory, the bistable cells are first connected in a feedforward manner, getting rid of the network memory coming from global recurrent connections. Despite having only cellular temporal connections, we show that BRCs provide good performances on standard benchmarks and surpass more standard ones such as LSTMs and GRUs on benchmarks with datasets composed of extremely sparse time-series. Secondly, we show that the proposed bio-inspired recurrent cell can be made more comparable to a standard GRU by using a special kind of recurrent neuromodulation. We call this neuromodulated bistable recurrent cell nBRC, standing for neuromodulated BRC. The comparison between nBRCs and GRUs provides food-for-thought and is a step towards reconciling traditional gated recurrent units and biological plausibility.

\section{Recurrent neural networks and gated recurrent units}
 RNNs have been widely used to tackle many problems having a temporal structure. In such problems, the relevant information can only be captured by processing observations obtained during multiple time-steps. More formally, a time-series can be defined as $\mb{X} = [\mb{x}_0,\ldots,\mb{x}_T]$ with $T \in \mathcal{N}_0$ and $\mb{x}_i \in \mathcal{R}^n$. To capture time-dependencies, RNNs maintain a recurrent hidden state whose update depends on the previous hidden state and current observation of a time-series, making them dynamical systems and allowing them to handle arbitrarily long sequences of inputs. Mathematically, RNNs maintain a hidden state $\mb{h}_t = f(\mb{h}_{t-1},\mb{x}_t;\theta)$, where $\mb{h}_0$ is a constant and $\theta$ are the parameters of the network. In its most standard form, an RNN updates its state as follows:
\begin{equation}
    \mb{h}_t = g(U\mb{x}_t + W\mb{h}_{t-1})
    \label{eq:srnn}
\end{equation}
where $g$ is a standard activation function such as a sigmoid or a hyperbolic tangent. However, RNNs using Equation~\ref{eq:srnn} as the update rule are known to be difficult to train on long sequences due to vanishing (or, more rarely, exploding) gradient problems. To alleviate this problem, more complex recurrent update rules have been proposed, such as LSTMs (\cite{lstm}) and GRUs (\cite{gru}). These updates allow recurrent networks to be trained on much longer sequences by using gating principles. By way of illustration, the updates related to a gated recurrent unit are
\begin{align}
    \begin{cases}
        \mb{z}_t = \sigma(U_z\mb{x}_t + W_z\mb{h}_{t-1}) \\
        \mb{r}_t = \sigma(U_r\mb{x}_t + W_r\mb{h}_{t-1}) \\
        %\mb{\Tilde{h}}_t = \tanh(U_h\mb{x}_t + \mb{r}_t \odot W_h\mb{h}_{t-1}) \\
        \mb{h}_t = \mb{z}_t\odot\mb{h}_{t-1} + (1 - \mb{z}_t)\odot \tanh(U_h\mb{x}_t + \mb{r}_t \odot W_h\mb{h}_{t-1}) 
    \end{cases}
    \label{eq:gru}
\end{align}
where $\mb{z}$ is the update gate (used to tune the update speed of the hidden state with respect to new inputs) and $\mb{r}$ is the reset gate (used to reset parts of the memory).

%\begin{itemize}
%   \item GRU recurrent update
%    \begin{align}
%    \begin{cases}
%        \mb{z}_t = \sigma(U_z\mb{x}_t + W_z\mb{h}_{t-1}) \\
%        \mb{r}_t = \sigma(U_r\mb{x}_t + W_r\mb{h}_{t-1}) \\
%        \mb{\Tilde{h}}_t = \tanh(U_h\mb{x}_t + r \odot W_h\mb{h}_{t-1}) \\
%        \mb{y}_t = \mb{h}_t = (1 - \mb{z}_t)\odot\mb{\Tilde{h}}_t + \mb{z}_t\odot\mb{h}_{t-1}
%    \end{cases}
%    \end{align}\label{eq:gru}
%    \item LSTM recurrent update
%    \begin{align}
%    \begin{cases}
%        \mb{f}_t = \sigma(U_f\mb{x}_t + W_f\mb{y}_{t-1}) \\
%        \mb{i}_t = \sigma(U_i\mb{x}_t + W_i\mb{y}_{t-1}) \\
%        \mb{o}_t = \sigma(U_o\mb{x}_t + W_o\mb{y}_{t-1}) \\
%        \mb{\Tilde{c}}_t = \tanh(U_c\mb{x}_t + W_c\mb{y}_{t-1})\\
%        \mb{h}_t = \mb{f}_t \odot \mb{h}_{t-1} + \mb{i}_t \odot \mb{\Tilde{c}}_t \\
%        \mb{y}_t = \tanh(\mb{h}_t) \odot \mb{o}_t
%    \end{cases}
%    \end{align}\label{eq:lstm}
%\end{itemize}
\section{Neuronal bistability: a feedback viewpoint}
\label{sec:bio}
Biological neurons are intrinsically dynamical systems that can exhibit a wide variety of firing patterns. In this work, we focus on the control of bistability, which corresponds to the coexistence of two stable states at the neuronal level. Bistable neurons can switch between their two stable states in response to transient inputs (\cite{marder, drion15}), endowing them with a kind of never-fading cellular memory (\cite{marder}). 

Complex neuron firing patterns are often modeled by systems of ordinary differential equations (ODEs). Translating ODEs into an artificial neural network algorithm often leads to mixed results due to increased complexity and the difference in modeling language. Another approach to model neuronal dynamics is to use a control systems viewpoint \cite{drioncontrole}. In this viewpoint, a neuron is modeled as a set of simple building blocks connected using a multiscale feedback, or recurrent, interconnection pattern. 

A neuronal feedback diagram focusing on one time-scale, which is sufficient for bistability, is illustrated in Figure~\ref{fig:control_diagram}A. The block $1/(Cs)$ accounts for membrane integration, $C$ being the membrane capacitance and $s$ the complex frequency. The outputs from presynaptic neurons $V_{pre}$ are combined at the input level to create a synaptic current $I_{syn}$.  Neuron-intrinsic dynamics are modeled by the negative feedback interconnection of a nonlinear function $I_{int}=f(V_{post})$, called the IV curve in neurophysiology, which outputs an intrinsic current $I_{int}$ that adds to $I_{syn}$ to create the membrane current $I_m$. The slope of $f(V_{post})$ determines the feedback gain, a positive slope leading to negative feedback and a negative slope to positive feedback. $I_m$ is then integrated by the postsynaptic neuron membrane to modify its output voltage $V_{post}$.

%This control perspective allows to give an abstract, but correct, model of a biological neuron, which diagram is shown on Figure~\ref{fig:control_diagram}.

\begin{figure}[h]
    \centering
    \includegraphics[width = 0.9\textwidth]{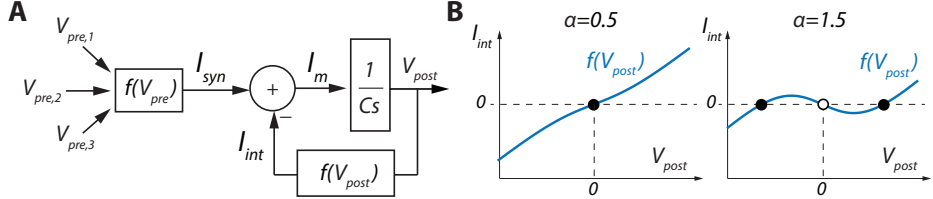}
    \caption{\textbf{A.} One timescale control diagram of a neuron. \textbf{B.} Plot of the function $I_{int} = V_{post} -\alpha \tanh(V_{post})$ for two different values of $\alpha$. Full dots correspond to stable states, empty dots to unstable states.}
    \label{fig:control_diagram}
\end{figure}

The switch between monostability and bistability is achieved by shaping the nonlinear function $I_{int}=f(V_{post})$ (Figure~\ref{fig:control_diagram}B). The neuron is monostable when $f(V_{post})$ is monotonic of positive slope (Figure~\ref{fig:control_diagram}B, left). Its only stable state corresponds to the voltage at which $I_{int}=0$ in the absence of synaptic inputs (full dot). The neuron switch to bistability through the creation of a local region of negative slope in $f(V_{post})$ (Figure~\ref{fig:control_diagram}B, left). Its two stable states correspond to the voltages at which $I_{int}=0$ with positive slope (full dots), separated by an unstable state where $I_{int}=0$ with negative slope (empty dot). The local region of negative slope corresponds to a local positive feedback where the membrane voltage is unstable.

In biological neurons, a local positive feedback is provided by regenerative gating, such as sodium and calcium channel activation or potassium channel inactivation (\cite{drion15, drion13}). The switch from monostability to bistability can therefore be controlled by tuning ion channel density. This property can be emulated in electrical circuits by combining transconductance amplifiers to create the function
\begin{equation}
    I_{int} = V_{post} -\alpha \tanh(V_{post}),
    \label{eq:bist}
\end{equation}
where the switch from monostability to bistability is controlled by a single parameter $\alpha$ (\cite{ribar}). $\alpha$ models the effect of sodium or calcium channel activation, which tunes the local slope of the function, hence the local gain of the feedback loop (Figure~\ref{fig:control_diagram}B). For $\alpha \in ]0,1]$, which models a low sodium or calcium channel density, the function is monotonic, leading to monostability (Figure~\ref{fig:control_diagram}B, left). For $\alpha \in ]1,+\infty[$, which models a high sodium or calcium channel density, a region of negative slope is created around $V_{post}=0$, and the neuron becomes bistable (Figure~\ref{fig:control_diagram}B, right). This bistability leads to never-fading memory, as in the absence of significant input perturbation the system will remain indefinitely in one of the two stable states depending on the input history.

Neuronal bistability can therefore be modeled by a simple feedback system whose dynamics is tuned by a single feedback parameter $\alpha$. This parameter can switch between monostability and bistability by tuning the shape of the feedback function $f(V_{post})$, whereas neuron convergence dynamics is controlled by a single feedforward parameter $C$. In biological neurons, both these parameters can be modified dynamically by other neurons via a mechanism called neuromodulation, providing a dynamic, controllable memory at the cellular level. The key challenge is to find an appropriate mathematical representation of this mechanism to be efficiently used in artificial neural networks, and, more particularly, in RNNs.

\section{Cellular memory, bistability and neuromodulation in RNNs}
\paragraph{The bistable recurrent cell (BRC)}
To model controllable bistability in RNNs, we start by drawing two main comparisons between the feedback structure Figure~\ref{fig:control_diagram}A and the GRU equations (Equation~\ref{eq:gru}). First, we note that the reset gate $r$ has a role that is similar to the one played by the feedback gain $\alpha$ in Equation~\ref{eq:bist}. In GRU equations, $r$ is the output of a sigmoid function, which implies $r \in ]0,1[$. These possible values for $r$ correspond to negative feedback only, which does not allow for bistability. The update gate $z$, on the other hand, has a role similar to that of the membrane capacitance $C$. Second, one can see through the matrix multiplications $W_z\mb{h}_{t-1}$, $W_r\mb{h}_{t-1}$ and $W_h\mb{h}_{t-1}$ that each cell uses the internal state of other neurons to compute its own state without going through synaptic connections. In biological neurons, the intrinsic dynamics defined by $I_{int}$ is constrained to only depend on its own state $V_{post}$, and the influence of other neurons comes only through the synaptic compartment ($I_{syn}$), or through neuromodulation.

%Despite being of great importance in biological brains, bistability and cellular memory are features which are not respected in GRU equations (Equation~\ref{eq:gru}). 
%\begin{itemize}
%    \item First, we note that GRU are bound to have a negative feedback, through the reset gate $r \in ]0,1[$. This constraints the cell to monostability.
%    \item Second, one can see that through the matrix multiplications $W_z\mb{h}_{t-1}$, $W_r\mb{h}_{t-1}$ and $W_h\mb{h}_{t-1}$ that each cell uses the internal state of other neurons to compute its own. This breaks the cellular memory constraint, as it would amount to neurons accessing $V_{post}$ of other neurons to compute their own $I_{int}$ (Figure~\ref{fig:control_diagram}).
%\end{itemize}

%The challenge is thus to equip GRU with bistability and cellular memory, while altering their update rule the least possible. This resulted in an updated cell, which we call bistable recurrent cell (BRC) and that uses the following update rule:

%To allow for bistability and cellular memory in GRU equations, we modify the two aforementioned properties to bring the RNN equations closer to the diagram of (Figure~\ref{fig:control_diagram}A). It results in an updated cell, which we call bistable recurrent cell (BRC) and that uses the following update rule:
To enforce this cellular feedback constraint in GRU equations and to endow them with bistability, we propose to update $h_t$ as follows: 
\begin{align}
        \mb{h}_t = \mb{c}_t\odot\mb{h}_{t-1} + (1-\mb{c}_t)\odot \tanh(U\mb{x}_t + \mb{a}_t \odot \mb{h}_{t-1})
    \label{eq:brc}
\end{align}
where $\mb{a}_t = 1 + \tanh(U_a\mb{x}_t + \mb{w}_a\odot\mb{h}_{t-1})$ and $\mb{c}_t = \sigma(U_{c}\mb{x}_t + \mb{w}_{c}\odot\mb{h}_{t-1})$. $\mb{a}_t$ corresponds to the feedback parameter $\alpha$, with $\mb{a}_t \in ]0,2[$. $\mb{c}_t$ corresponds to the update gate in GRU and plays the role of the membrane capacitance $C$, determining the convergence dynamics of the neuron. We call this updated cell the bistable recurrent cell (BRC). 

The main differences between a BRC and a GRU are twofold. First, each neuron has its own internal state $\mb{h}_{t}$ that is not directly affected by the internal state of the other neurons. Indeed, due to the four instances of $\mb{h}_{t-1}$ coming from Hadamard products, the only temporal connections existing in layers of BRC are from neurons to themselves. This enforces the memory to be only cellular. Second, the feedback parameter $\mb{a}_t$ is allowed to take a value in the range $]0,2[$ rather than $]0,1[$. This allows the cell to switch between monostability ($a \leq 1$) and bistability ($a > 1$) (Figure~\ref{fig:brc}A,B). The proof of this switch is provided in Appendix~\ref{app:bist_proof}.

It is important to note that the parameters $\mb{a}_t$ and $\mb{c}_t$ are dynamic. $\mb{a}_t$ and $\mb{c}_t$ are neuromodulated by the previous layer, that is, their value depends on the output of other neurons. Tests were carried with $\mb{a}$ and $\mb{c}$ as parameters learned by stochastic gradient descent, which resulted in lack of representational power, leading to the need for neuromodulation. This neuromodulation scheme was the most evident as it maintains the cellular memory constraint and leads to the most similar update rule with respect to standard recurrent cells (Equation~\ref{eq:gru}). However, as will be discussed later, other neuromodulation schemes can be thought of. 

Likewise, from a neuroscience perspective, $\mb{a}_t$ could well be greater than $2$. Limiting the range of $\mb{a}_t$ to $]0,2[$ was made for numerical stability and for symmetry between the range of bistable and monostable neurons. We argue that this is not an issue as, for a value of $\mb{a}_t$ greater than $1.5$, the dynamics of the neurons become very similar (as suggested in Figure~\ref{fig:brc}A).
\begin{figure}[h]
\begin{center}
    \includegraphics[width=0.75\textwidth]{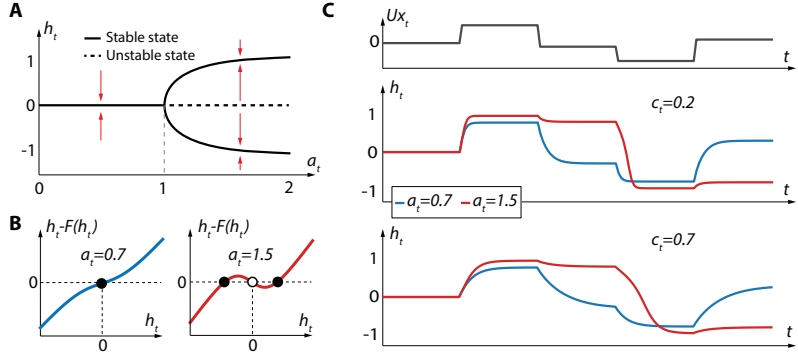}
\end{center}
\caption{\textbf{A.} Bifurcation diagram of Equation~\ref{eq:brc} for $U\mb{x}_t = 0$.~\textbf{B.} Plots of the function $h_t-F(h_t)$ for two values of $a_t$, where $F(h_t)= c_t h_{t} + (1-c_t)\tanh(a_th_{t})$ is the right hand side of Equation~\ref{eq:brc} with $x_t=0$. Full dots correspond to stable states, empty dots to unstable states.~\textbf{C.} Response of BRC to an input time-series for different values of $a_t$ and $c_t$.}
\label{fig:brc}
\end{figure}

Figure~\ref{fig:brc}C shows the dynamics of a BRC with respect to $a_t$ and $c_t$. For $a_t<1$, the cell exhibits a classical monostable behavior, relaxing to the $0$ stable state in the absence of inputs (blue curves in Figure~\ref{fig:brc}C). On the other hand, a bistable behavior can be observed for $a_t>1$: the cells can either stabilize on an upper stable state or a lower stable state depending on past inputs (red curves in Figure~\ref{fig:brc}C). %, and the amplitude of these states depends on the value of $a_t$.
Since these upper and lower stable states do not correspond to an $h_t$ which is equal to $0$, they can be associated with cellular memory that never fades over time.
Furthermore, Figure~\ref{fig:brc} also illustrates that neuron convergence dynamics depend on the value of $c$.

%as, for values of $a$ greater than $1$ and in the absence of new inputs, the system stabilizes on a value greater than $0$ or lower than $0$ respectively if the last input was greater or lower than $0$. In opposition, the linear regime can also be observed when for values of $a$ smaller than $1$, where the network stabilises at $0$, regardless of the previous inputs. Furthermore, Figure~\ref{fig:brc} also well illustrates that the neuron reacts slower to new inputs for greater value of $e$.

\paragraph{The recurrently neuromodulated bistable recurrent cell (nBRC)}
To further improve the performance of BRC, one can relax the cellular memory constraint. By creating a dependency of $a_t$ and $c_t$ on the output of other neurons of the layer, one can build a kind of recurrent layer-wise neuromodulation. We refer to this modified version of a BRC as an nBRC, standing for recurrently neuromodulated BRC. The update rule for the nBRC is the same as for BRC, and follows Equation~\ref{eq:brc}. The difference comes in the computation of $a_t$ and $c_t$, which are neuromodulated as follows:
\begin{align}
    \begin{cases}
        \mb{a}_t = 1 + \tanh(U_a\mb{x}_t + W_a\mb{h}_{t-1}) \\
        \mb{c}_t = \sigma(U_c\mb{x}_t + W_c\mb{h}_{t-1}) \\
    \end{cases}
    \label{eq:bgru}
\end{align}
The update rule of nBRCs being that of BRCs (Equation~\ref{eq:brc}), bistable properties are maintained and hence the possibility of a cellular memory that does not fade over time. However, the new recurrent neuromodulation scheme adds a type of network memory on top of the cellular memory. %As the parameters $a_t$ and $c_t$ now depend on the internal states of other neurons $h_{t-1}$, recurrent neuromodulatory connections are now created within the layer.
This recurrent neuromodulation scheme brings the update rule even closer to standard GRU. This is highlighted when comparing Equation~\ref{eq:gru} and Equation~\ref{eq:brc} with parameters neuromodulated following Equation~\ref{eq:bgru}. We stress that, as opposed to GRUs, bistability is still ensured through $a_t$ belonging to $]0,2[$. A relaxed cellular memory constraint is also ensured, as each neuron past state $h_{t-1}$ only directly influences its own current state and not the state of other neurons of the layer (Hadamard product on the $\mb{h}_t$ update in Equation~\ref{eq:brc}). This is important for numerical stability as the introduction of a cellular positive feedback for bistability leads to global instability if the update is computed using other neurons states directly (as it is done in the classical GRU update, see the matrix multiplication $W_h\mb{h}_{t-1}$ in Equation~\ref{eq:gru}).

%Despite this similarity, there are two main differences which are essential for introducing bistability and numerical stability:

%\begin{itemize}
 %   \item The output of the regret gate in standard GRU belongs to $]0,1[$, whereas in nBRC, $\mb{a}_t$ belongs to $]0,2[$. This limits standard GRU to a negative feedback regime, inevitably forgetting information in the absence of relevant input, whereas our model allows for a positive feedback and bistability.
  %  \item Each neuron past state $h_{t-1}$ only directly influences its own current state, not the state of other neurons of the layer. This represented by the Hadamard product on the $\mb{h}_t$ update in Equation~\ref{eq:brc}. This is extremely important for numerical stability. Indeed, the introduction of a cellular positive feedback for bistability, leads to global instability if the update is computed using other neurons states (as happens in GRU through the matrix multiplication $W_h\mb{h}_{t-1}$ in Equation~\ref{eq:gru}).
%\end{itemize}
Finally, let us note that to be consistent with the biological model presented in Section~\ref{sec:bio}, Equation~\ref{eq:bgru} should be interpreted as a way to represent a neuromodulation mechanism of a neuron by those from its own layer and the layer that precedes. Hence, the possible analogy between gates $z$ and $r$ in GRUs and neuromodulation. In this respect, studying the introduction of new types of gates based on more biological plausible neuromodulation architectures would certainly be interesting.

%By comparing Equations~\ref{eq:gru}~and~\ref{eq:bgru}, one can make the analogy between gates ($\mb{z}_t$ and $\mb{r}_t$ in GRU) and neuromodulation (the way neuronal parameters such as $\mb{a}_t$ and $\mb{c}_t$ are dynamically modified by other neurons). Studying the introduction of new types of gates based on more biologically plausible neuromodulation architectures would be a very interesting subject and is left as future work. %Here, we highlight this comparison, providing insights on how one can link recurrent gates to neuromodulation.

\section{Analysis of BRC and nBRC performance}

To demonstrate the performances of BRCs and nBRCs with respect to standard GRUs and LSTMs, we tackle three problems. The first is a one-dimensional toy problem, the second is a two-dimensional denoising problem and the third is the sequential MNIST problem. The supervised setting used is the same for all three benchmarks. The network is presented with a time-series and is asked to output a prediction (regression for the first two benchmarks and classification for the third) after having received the last element of the time-series $\mb{x}_T$. We show that the introduction of bistability in recurrent cells is especially useful for datasets in which only sparse time-series are available. In this section, we also take a look at the dynamics inside the BRC neurons in the context of the denoising benchmark and show that bistability is heavily used by the neural network.

\subsection{Results}
For the first two problems, training sets comprise $45000$ samples and performances are evaluated on test sets generated with $50000$ samples. For the MNIST benchmark, the standard train and test sets are used. All averages and standard deviations reported were computed over three different seeds. We found that there were very little variations in between runs, and thus believe that three runs are enough to capture the performance of the different architectures. For benchmark 1, networks are composed of two recurrent layers of $100$ neurons each whereas for benchmark 2 and 3, networks are composed of four recurrent layers of $100$ neurons each. Different recurrent cells are always tested on similar networks (i.e. same number of layers/neurons). We used the tensorflow (\cite{tensorflow}) implementation of GRUs and LSTMs. Finally, the ADAM optimizer with a learning rate of $1e^{-3}$ is used for training all networks, with a mini-batch size of $100$. The source code for carrying out the experiments is available at \url{https://github.com/nvecoven/BRC}.

\paragraph{Copy first input benchmark}
In this benchmark, the network is presented with a one-dimensional time-series of $T$ time-steps where $x_t \sim \mathcal{N}(0,1), ~\forall t \in T$. After receiving $x_T$, the network output value should approximate $x_0$, a task that is well suited for capturing their capacity to learn long temporal dependencies if $T$ is large. Note that this benchmark also requires the ability to filter irrelevant signals as, after time-step $0$, the networks are continuously presented with noisy inputs that they must learn to ignore. The mean square error on the test set is shown for different values of $T$ in Table~\ref{table:bench1}. For smaller values of $T$, all recurrent cells achieve similar performances. The advantage of using bistable recurrent cells appears when $T$ becomes large (Figure~\ref{fig:learning_dynamics}). Indeed, when $T$ is equal to $600$, only networks made of bistable cells are capable of outperforming random guessing threshold (which would be equal to $1$ in this setting \footnote{As $x_0$ is sampled from a normal distribution $\mathcal{N}(0,1)$, guessing $0$ would lead to the lowest error which would on average be equal to the standard deviation.}).

\begin{table}
\begin{center}
\begin{tabular}{|c||c|c||c|c|}
\hline
T & BRC & nBRC & GRU & LSTM \\ \hline\hline
$5$ & $0.0157 \pm 0.0124$& $0.0028 \pm 0.0023$& $0.0019 \pm 0.0011$& $\mb{0.0016 \pm 0.0009}$\\ \hline 
$50$ & $0.0142 \pm 0.0081$& $0.0009 \pm 0.0006$ &$\mb{0.0004 \pm 0.0003}$& $0.9919 \pm 0.0012$\\ \hline 
$100$ & $0.0046 \pm 0.0002$& $\mb{0.0006 \pm 0.0001}$& $0.0097 \pm 0.006$& $1.0354 \pm 0.0416$\\ \hline 
$300$ & $0.0013 \pm 0.0008$ &$\mb{0.0007 \pm 0.0002}$& $0.6743 \pm 0.4761$&$0.9989 \pm 0.0170$\\ \hline 
$600$ & $0.1581 \pm 0.1574$ & $\mb{0.0005 \pm 0.0001}$ &$0.9934 \pm 0.0182$ & $0.9989\pm 0.0162$\\ \hline 
\end{tabular}\end{center}\caption{Mean square error on test set after $30000$ gradient descent iterations of different architectures on the copy first input benchmark. Results are shown for different values of $T$.}\label{table:bench1}\end{table}

\begin{figure}
    \centering
    \includegraphics[width=0.7\textwidth]{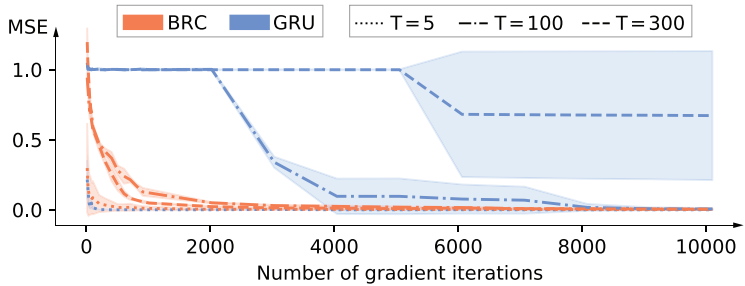}
    \caption{Evolution of the average mean-square error ($\pm$ standard deviation) over three runs on the copy input benchmark for GRU and BRC and for different values of $T$.}
    \label{fig:learning_dynamics}
\end{figure}

\paragraph{Denoising benchmark}
The copy input benchmark is interesting as a means to highlight the memorisation capacity of the recurrent neural network, but it does not tackle its ability to successfully exploit complex relationships between different elements of the input signal to predict the output. In the denoising benchmark, the network is presented with a two-dimensional time-series of $T$ time-steps. Five different time-steps $t_1,\ldots,t_5$ are sampled uniformly in $\{0,\ldots,T-N\}$ with $N \in \{0,\ldots,T-5\}$ and are communicated to the network through the first dimension of the time-series by setting $\mb{x}_t[1] = 0$ if $t\in[t_1,\ldots,t_5]$, $\mb{x}_t[1] = 1$ if $t=T$ and $\mb{x}_t[1] = -1$ otherwise.

Note that this dimension is also used to notify the network that the end of the time-series is reached (and thus, that the network should output its prediction). The second dimension is a data-stream, generated as for the copy first input benchmark, that is $\mb{x}_t[2] \sim \mathcal{N}(0,1), \forall t \in T$. At time-step $T$, the network is asked to output $[\mb{x}_{t_1}[2],\ldots,\mb{x}_{t_5}[2]]$.
The mean squared error is averaged over the $5$ values. That is, the error on the prediction is equal to $\sum_{i=1}^{5}\frac{(\mb{x}_{t_i}[2]-\mb{O}[i])^2}{5}$ with $\mb{O}$ the output of the neural network. Note that the parameter $N$ controls the length of the forgetting period as it forces the relevant inputs to be in the first $T-N$ time-steps. This ensures that $t_x < T-N,~\forall x \in \{1,\ldots,5\}$.

As one can see in Table~\ref{table:bench2} (generated with $T = 200$ and two different values of $N$), for $N = 200$, GRUs and LSTMs are unable to exceed random guessing (mean square error of $1$) whereas BRC and nBRC performances are virtually unimpacted. Also, Table~\ref{table:bench2} provides a very important observation. GRUs and LSTMs are, in fact, able to learn long-term dependencies, as they achieve extremely good performances when $N = 0$. In fact, all the samples generated when $N = 200$ could also be generated when $N = 0$, meaning that with the right parameters, the GRUs and LSTMs network could achieve good predictive performances on such samples. However, our results show that GRUs and LSTMs are unable to learn those parameters when the datasets are only composed of such samples. That is, GRUs and LSTMs need training datasets with some samples for which the memory required is quite short to learn efficiently, and allow for learning the samples for which the temporal structure is longer. Bistable cells on the other hand are not susceptible to this caveat.

\begin{table}
\begin{center}\begin{tabular}{|c||c|c||c|c|}
\hline
$N$ & BRC & nBRC & GRU & LSTM \\ \hline\hline
$0$ & $0.0014 \pm 0.0001$ & $0.0006 \pm 0.0001$ & $\mb{0.0001 \pm 0.0001}$ & $0.0003 \pm 0.0002$\\ \hline 
$200$ & $0.0032 \pm 0.0015$ & $\mb{0.0013 \pm 0.0006}$&$1.0571 \pm 0.0452$& $0.9878 \pm 0.0052$\\ \hline 
\end{tabular}\end{center}\caption{Mean square error on test set after $30000$ gradient descent iterations of different architectures on the denoising benchmark. Results are shown with and without constraint on the location of relevant inputs. Relevant inputs cannot appear in the $N$ last time-steps, that is $\mb{x}_t[1] = -1, \forall t > (T-N)$. In this experiment, results were obtained with $T = 400$.}\label{table:bench2}\end{table}

To further highlight this behavior, we design another benchmark that is a variant of the copy input benchmark. In this benchmark, the network is presented with a one-dimensional time-series of length $T = 600$ where $x_t = 0,~\forall t \in T \backslash t_1$ and $x_{t_1} \sim ~ \mathcal{N}(0,1)$, with $t_1$ chosen uniformly in $\{0,\ldots,599\}$. The network is tasked to output $x_{t_1}$.
Table~\ref{table:rand} shows that, using this training scenario, GRUs are capable of achieving a low MSE (around $0.04$) on the test set used for the original copy input benchmark in which all the $T = 600$. This was not the case in Table~\ref{table:bench1} (MSE around $1.0$ for $T = 600$), when trained on a datasets for which all the samples require a $600$ time-step-long dependency. On the other hand, the performance of BRC and nBRC peaked in this scenario. 

\begin{table}
\begin{center}\begin{tabular}{|c|c||c|c|}
\hline
BRC & nBRC & GRU & LSTM \\ \hline
$0.0010 \pm 0.0001$& $\mb{0.0001 \pm 0.0001}$&$0.0373 \pm 0.00371$& $0.3323 \pm 0.4635$\\ \hline 
\end{tabular}\end{center}\caption{Mean square error on test set after $30000$ gradient descent iterations of different architectures on the modified copy input benchmark.}\label{table:rand}\end{table}

\paragraph{Sequential MNIST}

In this benchmark, the network is presented with the MNIST images, shown pixel by pixel as a time-series. MNIST images are made of 1024 pixels (32 by 32), showing that BRC and nBRC can learn dynamics over thousands of time-steps. Similar to both previous benchmarks, we add $n_{black}$ time-steps of black pixels at the end of the time-series to add a forgetting period. Results are shown in Table~\ref{table:mnist} for two values of $n_{black}$, and are consistent with what has been observed in both previous benchmarks.

\begin{table}
\begin{center}\begin{tabular}{|c||c|c||c|c|}
\hline
$n_{black}$ & BRC & nBRC & GRU & LSTM \\ \hline\hline
$0$ & $0.9697 \pm 0.0020$ & $0.9601 \pm 0.0032$ &$\mb{0.9880 \pm 0.0014}$& $0.9651 \pm 0.0023$\\ \hline 
$300$ & $\mb{0.9760 \pm 0.0015}$ & $0.9608 \pm 0.0109$& $0.1081 \pm 0.0053$&$0.1124 \pm 0.0013$\\ \hline 
\end{tabular}\end{center}\caption{Accuracy on MNIST test set after $30000$ gradient descent iterations on different architectures on the MNIST benchmark. Images are fed to the recurrent network pixel by pixel. Results are shown for MNIST images with $n_{black}$ black pixels appended to the image.}\label{table:mnist}\end{table}

\subsection{Analysis of BRC dynamic behavior}
Until now, we have looked at the learning performances of bistable recurrent cells. It is however interesting to take a deeper look at the dynamics of such cells to understand whether or not bistability is used by the network. To this end, we pick a random time-series from the denoising benchmark and analyse some properties of $a_t$ and $c_t$. Figure~\ref{fig:dynamics} shows the proportion of bistable cells per layer and the average value of $e_t$ per layer. The dynamics of the parameters show that they are well used by the network, and three main observations owe to be made. First, as relevant inputs are shown to the network, the proportion of bistable neurons tends to increase in layers 2 and 3, effectively storing information and thus confirming the interest of introducing bistability for long-term memory. As more information needs to be stored, the network leverages the power of bistability by increasing the number of bistable neurons. Second, as relevant inputs are shown to the network, the average value $c_t$ tends to increase in layer 3, effectively making the network less and less sensitive to new inputs. Third, one can observe a transition regime when a relevant input is shown. Indeed, there is a high decrease in the average value of $c_t$, effectively making the network extremely sensitive to the current input, which allows for its efficient memorization.

\begin{figure}
\centering
%\begin{subfigure}{.5\textwidth}
  %\centering
  \includegraphics[width=1\linewidth]{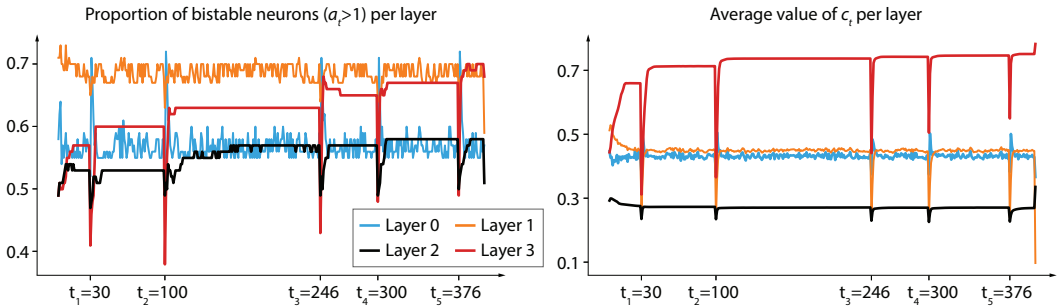}
  %\caption{Proportion of neurons that are bistable ($a_t > 1$) per layer.}
  %\label{fig:sub1}
%\end{subfigure}%
%\begin{subfigure}{.5\textwidth}
  %\centering
  %\includegraphics[width=1\linewidth]{params_z_analysis.pdf}
  %\caption{Average value of $c_t$ per layer}
  %\label{fig:sub2}
%\end{subfigure}
\caption{Representation of the BRC parameters, per layer, of a recurrent neural network (with $4$ layers of $100$ neurons each), when shown a time-series of the denoising benchmark ($T = 400$, $N = 0$). Layer numbering increases as layers get deeper (i.e. layer i corresponds to the ith layer of the network). The $5$ time-steps at which a relevant input is shown to the model are clearly distinguishable by the behaviour of those measures alone.}
\label{fig:dynamics}
\end{figure}

\section{Conclusion}
In this paper, we introduced two new important concepts from the biological brain into recurrent neural networks: cellular memory and bistability. %These new concepts proved to be extremely beneficial for datasets where there exists no sample requiring short memory, allowing to learn on datasets where only samples requiring extremely long memory are available. Also, for more standard benchmarks where short-memory samples are available, we showed that the lack of representational power introduced with cellular memory can easily be compensated by recurrent neuromodulation. The performances obtained are similar to other, supposedly more representative recurrent cells, such a GRUs and LSTMs.
This lead to the development of a new cell, called the Bistable Recurrent Cell (BRC) that proved to be very efficient on several datasets requiring long-term memory and on which the performances of classical recurrent cells such as GRUs and LSTMS were limited.

Furthermore, by relaxing the cellular memory constraint and using a special rule for recurrent neuromodulation, we were able to create a neuromodulated bistable recurrent cell (nBRC) which is very similar to a standard GRU. This is of great interest and provides insights on how gates in GRUs and LSTMs, among others, could in fact be linked to what is neuromodulation in biological brains. As future work, it would be of interest to study some more complex and biologically plausible neuromodulation schemes and see what types of new, gated architectures could emerge from them.

\subsection*{Acknowledgements}
Nicolas Vecoven gratefully acknowledges the financial support of the Belgian FRIA.

\newpage

\bibliographystyle{vancouver}
\bibliography{sample}

\begin{thebibliography}{10}

\bibitem{lstm}
Hochreiter S, Schmidhuber J.
\newblock Long short-term memory.
\newblock Neural computation. 1997;9(8):1735--1780.

\bibitem{gru}
Cho K, Van~Merri{\"e}nboer B, Bahdanau D, Bengio Y.
\newblock On the properties of neural machine translation: Encoder-decoder
  approaches.
\newblock arXiv preprint arXiv:14091259. 2014;.

\bibitem{rnnevaluation}
Chung J, Gulcehre C, Cho K, Bengio Y. Empirical Evaluation of Gated Recurrent
  Neural Networks on Sequence Modeling; 2014.

\bibitem{control_crnn}
Sussillo D, Barak O.
\newblock Opening the black box: low-dimensional dynamics in high-dimensional
  recurrent neural networks.
\newblock Neural computation. 2013;25(3):626--649.

\bibitem{mgu}
Zhou GB, Wu J, Zhang CL, Zhou ZH.
\newblock Minimal gated unit for recurrent neural networks.
\newblock International Journal of Automation and Computing.
  2016;13(3):226--234.

\bibitem{gruvariants}
Dey R, Salemt FM.
\newblock Gate-variants of gated recurrent unit (GRU) neural networks.
\newblock In: 2017 IEEE 60th international midwest symposium on circuits and
  systems (MWSCAS). IEEE; 2017. p. 1597--1600.

\bibitem{gated_orthogonal}
Jing L, Gulcehre C, Peurifoy J, Shen Y, Tegmark M, Soljacic M, et~al.
\newblock Gated orthogonal recurrent units: On learning to forget.
\newblock Neural computation. 2019;31(4):765--783.

\bibitem{empiricalrnn}
Jozefowicz R, Zaremba W, Sutskever I.
\newblock An empirical exploration of recurrent network architectures.
\newblock In: International conference on machine learning; 2015. p.
  2342--2350.

\bibitem{spikingdeep}
Tavanaei A, Ghodrati M, Kheradpisheh SR, Masquelier T, Maida A.
\newblock Deep learning in spiking neural networks.
\newblock Neural Networks. 2019;111:47--63.

\bibitem{spikingdeep2}
Pfeiffer M, Pfeil T.
\newblock Deep learning with spiking neurons: opportunities and challenges.
\newblock Frontiers in neuroscience. 2018;12:774.

\bibitem{lstminspiking}
Bellec G, Salaj D, Subramoney A, Legenstein R, Maass W.
\newblock Long short-term memory and learning-to-learn in networks of spiking
  neurons.
\newblock In: Advances in Neural Information Processing Systems; 2018. p.
  787--797.

\bibitem{towardsplausible}
Bengio Y, Lee DH, Bornschein J, Mesnard T, Lin Z.
\newblock Towards biologically plausible deep learning.
\newblock arXiv preprint arXiv:150204156. 2015;.

\bibitem{plausiblernnlearning}
Miconi T.
\newblock Biologically plausible learning in recurrent neural networks
  reproduces neural dynamics observed during cognitive tasks.
\newblock Elife. 2017;6:e20899.

\bibitem{biologicallyinspiredbptt}
Bellec G, Scherr F, Hajek E, Salaj D, Legenstein R, Maass W.
\newblock Biologically inspired alternatives to backpropagation through time
  for learning in recurrent neural nets.
\newblock arXiv preprint arXiv:190109049. 2019;.

\bibitem{recurrentneuroscience}
Barak O.
\newblock Recurrent neural networks as versatile tools of neuroscience
  research.
\newblock Current opinion in neurobiology. 2017;46:1--6.

\bibitem{marder}
Marder E, Abbott L, Turrigiano GG, Liu Z, Golowasch J.
\newblock Memory from the dynamics of intrinsic membrane currents.
\newblock Proceedings of the national academy of sciences.
  1996;93(24):13481--13486.

\bibitem{drioncontrole}
Drion G, O'Leary T, Dethier J, Franci A, Sepulchre R.
\newblock Neuronal behaviors: A control perspective.
\newblock In: 2015 54th IEEE Conference on Decision and Control (CDC). IEEE;
  2015. p. 1923--1944.

\bibitem{drion15}
Drion G, O’Leary T, Marder E.
\newblock Ion channel degeneracy enables robust and tunable neuronal firing
  rates.
\newblock Proceedings of the National Academy of Sciences.
  2015;112(38):E5361--E5370.

\bibitem{drion13}
Franci A, Drion G, Seutin V, Sepulchre R.
\newblock A balance equation determines a switch in neuronal excitability.
\newblock PLoS computational biology. 2013;9(5).

\bibitem{ribar}
Ribar L, Sepulchre R.
\newblock Neuromodulation of neuromorphic circuits.
\newblock IEEE Transactions on Circuits and Systems I: Regular Papers.
  2019;66(8):3028--3040.

\bibitem{tensorflow}
Abadi M, Agarwal A, Barham P, Brevdo E, Chen Z, Citro C, et~al.. {TensorFlow}:
  Large-Scale Machine Learning on Heterogeneous Systems; 2015.
\newblock Software available from tensorflow.org.
\newblock Available from: \url{http://tensorflow.org/}.

\bibitem{golubitsky2012}
Golubitsky M, Schaeffer DG.
\newblock Singularities and groups in bifurcation theory. vol.~1.
\newblock Springer Science \& Business Media; 2012.

\end{thebibliography}
\newpage
\appendix
\section{Proof of bistability for BRC and nBRC for $a_t>1$}
\label{app:bist_proof}

\begin{theorem}
The system defined by the equation
\begin{equation}
h_t = c h_{t-1} + (1 - c)\tanh(Ux_t+ah_{t-1}) = F(h_{t-1})
\end{equation}
with $c \in [0,1]$ is monostable for $a \in [0,1[$ and bistable for $a>1$ in some finite range of $Ux_t$ centered around $x_t=0$.
\end{theorem}

\begin{proof}
We can show that the system undergoes a supercritical pitchfork bifurcation at the equilibrium point $(x_0,h_0)=(0,0)$ for $a=a_{pf}=1$ by verifying the conditions
\begin{eqnarray}
&G(h_{0})\big\rvert_{a_{pf}}=\frac{dG(h_{t})}{dh_t}\big\rvert_{h_{0},a_{pf}}=\frac{d^2G(h_{t})}{dh_t^2}\big\rvert_{h_{0},a_{pf}}=\frac{dG(h_{t})}{da}\big\rvert_{h_{0},a_{pf}} = 0\\
&\frac{d^3G(h_{t})}{dh_t^3}\big\rvert_{h_{0},a_{pf}} > 0, \frac{d^2G(h_{t})}{dh_tda}\big\rvert_{h_{0},a_{pf}} < 0
\end{eqnarray}
where $G(h_t) = h_t - F(h_t)$ (\cite{golubitsky2012}). This gives
\begin{eqnarray}
G(h_{0})\big\rvert_{a_{pf}}&=&(1-c)(h_0-\tanh(a_{pf}h_0)) = 0,\\
\frac{dG(h_{t})}{dh_t}\bigg\rvert_{h_{0},a_{pf}} &=& (1-c)(a_{pf}(\tanh^2(a_{pf}h_0)-1)+1) = (1-c)(1-a_{pf})=0,\\
\frac{d^2G(h_{t})}{dh_t^2}\bigg\rvert_{h_{0},a_{pf}} &=& (1-c)2a_{pf}^2\tanh(a_{pf}h_0)(1-\tanh^2(a_{pf}h_0))=0,\\
\frac{dG(h_{t})}{da}\bigg\rvert_{h_{0},a_{pf}} &=& (1-c)h_0(\tanh(a_{pf}h_0)^2 - 1)=0,\\
\frac{d^3G(h_{t})}{dh_t^3}\bigg\rvert_{h_{0},a_{pf}} &=& (1-c)*(2a^3(\tanh^2(a_{pf}h_0)-1)^2+4a_{pf}^3\tanh^2(a_{pf}h_0)(\tanh^2(a_{pf}h_0)-1)) \nonumber \\
 &=& 2(1-c) > 0,\\
\frac{d^2G(h_{t})}{dh_tda}\bigg\rvert_{h_{0},a_{pf}} &=& (1-c)((\tanh^2(a_{pf}h_0)-1)+2a_{pf}h_0\tanh(a_{pf}h_0)(1-\tanh^2(a_{pf}h_0)))\nonumber \\
&=& c-1 < 0.
\end{eqnarray}
The stability of $(x_0,h_0)$ for $a\neq1$ can be assessed by studying the linearized system
\begin{eqnarray}
h_t = \frac{dF(h_{t})}{dh_t}\bigg\rvert_{h_{0}} h_{t-1}. 
\end{eqnarray}
The equilibrium point is stable if $dF(h_{t})/dh_t \in [0,1[$, singular if $dF(h_{t})/dh_t=1$, and unstable if $dF(h_{t})/dh_t \in ]1,+\infty[$. We have
\begin{eqnarray}
\frac{dF(h_{t})}{dh_t}\bigg\rvert_{h_{0}} &=& c  + (1-c)a(1-\tanh^2(a_th_{0}))\\
&=& c + (1-c)a,
\end{eqnarray}
which shows that $(x_0,h_0)$ is stable for $a \in [0,1[$ and unstable for $a>1$.

It follows that for $a<1$, the system has a unique stable equilibrium point at ($x_0,h_0$), whose uniqueness is verified by the monotonicity of $G(h_{t})$ ($dG(h_{t})/dh_t>0 \forall h_t$).

For $a>1$, the point ($x_0,h_0$) is unstable, and there exist two stable points ($x_0,\pm h_1$) whose basins of attraction are defined by $h_t \in ]-\infty,h_0[$ for $-h_1$ and $h_t \in ]h_0,+\infty[$ for $h_1$.
\end{proof}
\end{document}